\documentclass[conference]{IEEEtran}
\IEEEoverridecommandlockouts
\usepackage{cite}
\usepackage{amsmath,amssymb,amsfonts}
\usepackage{graphicx}
\usepackage{mdframed}
\usepackage{algorithm}
\usepackage{subfigure}

\usepackage{algorithmic}
\usepackage{amsmath}
\usepackage{pifont}
\usepackage{enumitem}
\usepackage{multirow}
\usepackage{booktabs}
\usepackage{makecell}
\usepackage{textcomp}
\usepackage[colorlinks,
            linkcolor=blue,       
            anchorcolor=blue,  
            citecolor=blue,        
            ]{hyperref}
\usepackage{xcolor}
\usepackage{amsmath,amssymb,amsthm}
\newtheorem{definition}{Definition} 
\newtheorem{theorem}{\bf Theorem}

\newtheorem{proposition}{\bf Proposition}
\def\BibTeX{{\rm B\kern-.05em{\sc i\kern-.025em b}\kern-.08em
    T\kern-.1667em\lower.7ex\hbox{E}\kern-.125emX}}
\begin{document}


\title{\textit{Trust-But-Verify}: Poisoning-Resilient Locally\\Private Graph Learning Protocols
}

\author{\IEEEauthorblockN{Longzhu He\textsuperscript{$\S$}, Li Sun\textsuperscript{$\S$}, Hao Peng\textsuperscript{$\dag$}, Ruijie Wang\textsuperscript{$\dag$},  Raymond Chi-Wing Wong\textsuperscript{\textbardbl}, Sen Su\textsuperscript{$\S$,$\ddag$}}
\IEEEauthorblockA{\textsuperscript{$\S$}Beijing University of Posts and Telecommunications\;\textsuperscript{$\dag$}Beihang University\\\textsuperscript{\textbardbl}The Hong Kong University of Science and Technology\;\textsuperscript{$\ddag$}Chongqing University of Posts and Telecommunications}
\texttt{\{helongzhu,lsun,susen\}@bupt.edu.cn; \{penghao,ruijiew\}@buaa.edu.cn; raywong@cse.ust.hk}
}

\maketitle

\begin{abstract}
Built upon local differential privacy (LDP), locally private graph learning protocols have emerged as an important paradigm for decentralized graph learning, balancing privacy protection and learning utility. Under such protocols, each user locally perturbs their node features and adjacency information before transmission, ensuring formal privacy guarantees without original data leaving the device. However, the inherently open participation nature renders these protocols critically vulnerable to data poisoning attacks, where adversaries inject carefully crafted malicious nodes to corrupt neighborhood aggregation and degrade downstream utility. Despite the severity of this threat, effective defenses in this setting remain largely unexplored. In this paper, we propose \textsc{Veritas}, a poisoning-resilient locally private graph learning protocol built on a \textit{trust-but-verify} paradigm. By introducing a verification list encoding graded peer trust levels, \textsc{Veritas} jointly privatizes node features and graph structure on the user side, while exploiting bilateral attestation asymmetry on the server side to identify and prune malicious nodes. Concretely, \textsc{Veritas} comprises four synergistic stages: \ding{172} \textit{local data perturbation}, \ding{173} \textit{attestation-driven malicious node pruning}, \ding{174} \textit{utility restoration via dual denoising}, and \ding{175} \textit{robust private graph learning}. Extensive experiments on four real-world benchmark datasets across multiple LDP mechanisms and GNN architectures  demonstrate that \textsc{Veritas} effectively defends against data poisoning attacks and significantly improves downstream graph learning utility under rigorous privacy guarantees.
\end{abstract}


\section{Introduction}

\begin{figure*}
  \centering
  \includegraphics[scale=0.7]{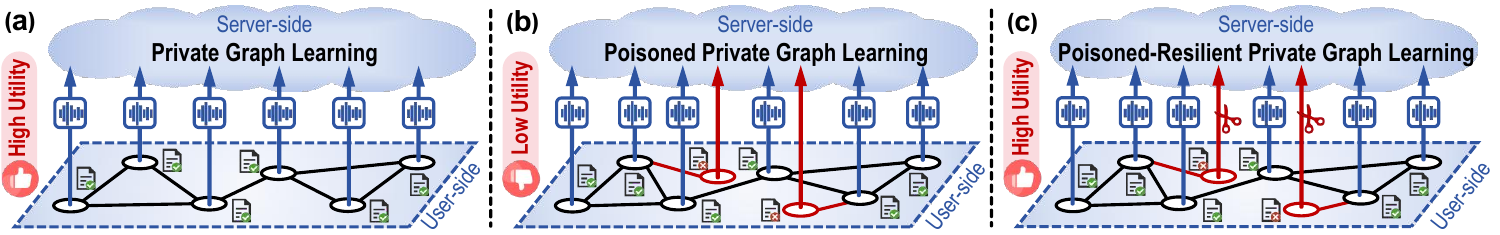}
  \caption{\textbf{Motivation and core intuition of \textsc{Veritas}.} \textbf{(a)} Locally private graph learning protocol. Multiple users independently perturb their sensitive node features and adjacency information via local differential privacy (LDP) mechanisms. The resulting noisy reports are transmitted to an untrusted server, which performs graph learning on the collected perturbed data. \textbf{(b)} Data poisoning attacks against locally private graph learning, where red nodes denote injected malicious participants that submit carefully crafted poisoned reports to corrupt the training process and degrade downstream task performance. \textbf{(c)} \textsc{Veritas} detects and prunes malicious nodes through bilateral attestation, preserving graph learning utility while maintaining rigorous local privacy guarantees.}
  \vspace{-0.5em}
  \label{fig1}
\end{figure*}

Graph neural networks (GNNs)~\cite{hamilton2017inductive,DBLP:conf/iclr/KipfW17,wu2020comprehensive} have emerged as a cornerstone tool in data mining, exhibiting strong representation learning capabilities across a wide range of graph mining tasks, including node classification~\cite{DBLP:conf/iclr/KipfW17}, link prediction~\cite{liu2025hierarchical}, and graph clustering~\cite{tsitsulin2023graph}. Their success has also enabled numerous real-world applications, such as recommendation systems~\cite{sharma2024survey}, social network analysis~\cite{jain2023opinion},  and fraud detection~\cite{liu2021pick}. However, the graph-structured data underlying these applications is often highly sensitive, with node features containing private user information and edges revealing confidential social relationships.
 Previous studies have demonstrated that GNN is vulnerable to privacy attacks, which can expose sensitive information from both node features and graph topology~\cite{meng2023devil,wu2022linkteller,wang2022group,zhang2022inference}. Such threats have raised growing concerns regarding data confidentiality and highlighted the urgent need for privacy-preserving graph learning techniques that effectively protect individual privacy while maintaining strong learning utility.

To address the growing privacy concerns in graph learning, \textit{locally private graph learning protocol}~\cite{sajadmanesh2021locally,lin2022towards,hegoing,he2025devil,hidano2022degree,li2024privacy,pei2023privacy,he2026devil,zhu2023blink,he2026differentially,he2026towards}, built upon local differential privacy (LDP)~\cite{kasiviswanathan2011can,dwork2008differential}, has attracted considerable attention from the security and privacy communities. Unlike centralized privacy-preserving approaches that require users to trust a data curator, LDP adopts a stronger trust model in which each user perturbs their sensitive data locally before any transmission, thereby providing rigorous privacy guarantees without relying on any trusted third party. Thanks to these appealing properties, LDP has been widely adopted by major technology companies, including Google~\cite{erlingsson2014rappor}, Apple~\cite{patent1}, and Microsoft~\cite{ding2017collecting}. As illustrated in Fig.~\ref{fig1}(a), each user independently perturbs their sensitive data before sending noisy reports to an untrusted server, which subsequently performs private graph learning on the collected noisy data. By ensuring that original user data never leaves local devices, this decentralized learning paradigm offers strong user-level privacy protection while retaining the utility of downstream graph learning tasks.

However, the open-participation nature of locally private graph learning renders it particularly vulnerable to \textit{data poisoning attacks}~\cite{he2025devil}, as illustrated in Fig.~\ref{fig1}(b). In realistic distributed scenarios, an adversary can create sybil accounts~\cite{douceur2002sybil}, establish plausible social identities~\cite{sun2022adversarial}, and infiltrate the graph by forming links with legitimate users~\cite{ma2020towards}. Once admitted, these malicious nodes submit carefully crafted perturbed feature and topology reports to the server, thereby injecting adversarial signals into the training data. Through iterative message passing, such poisoned information propagates through the graph and contaminates the learned node representations, ultimately degrading the performance of downstream tasks, such as node classification. Critically, local perturbation obscures the distinction between benign and malicious reports, making it extremely challenging for the server to identify adversarial participants based solely on the received data.

Despite the severity of this threat, effective defenses remain largely unexplored. Designing such defenses is highly non-trivial for three reasons. \textit{First}, the decentralized nature of LDP prevents the server from accessing raw features or adjacency information, rendering conventional anomaly detection and robust aggregation methods inapplicable. \textit{Second}, the noise introduced by LDP mechanisms obscures malicious manipulations, blurring the boundary between benign perturbation and adversarial injection, making it difficult to distinguish malicious nodes from honest ones based on noisy observations alone. \textit{Third}, existing graph poisoning defenses~\cite{zhang2020gnnguard,tang2020transferring,sun2022adversarial} are developed for centralized graph learning and assume access to unperturbed graph data, assumptions that are fundamentally violated under local privacy protection.

To address the above challenges, we develop a poisoning-resilient locally private graph learning protocol, termed \textsc{Veritas}\footnote{\textit{Veritas} is the Latin word for ``truth''.}, as shown in Fig.~\ref{fig1}(c). Following a \textit{trust-but-verify} paradigm, \textsc{Veritas} initially trusts all user-submitted reports while subsequently verifying their credibility through a verification list that encodes graded peer trust levels. This design enables the server to exploit the inherent asymmetry in bilateral attestation signals to identify and prune malicious nodes while preserving rigorous LDP guarantees on the user side. Concretely, our \textsc{Veritas} comprises four synergistic stages: \ding{172} \textit{local data perturbation} for privacy-preserving data collection, \ding{173} \textit{attestation-driven malicious node pruning} for verifying structural consistency and eliminating adversarial influence, \ding{174} \textit{utility restoration via dual denoising} to mitigate perturbation-induced noise, and \ding{175} \textit{robust private graph learning} for downstream graph learning. By seamlessly coupling these components, \textsc{Veritas} establishes an end-to-end defense pipeline that effectively counters data poisoning attacks while maintaining rigorous privacy guarantees. The main contributions of this paper are summarized as follows:
\begin{itemize}[leftmargin=*, itemindent=0em]
\renewcommand{\labelitemi}{$\diamond$}
\item \textbf{Important problem.} We investigate the problem of poisoning attacks against locally private graph learning, highlighting a critical challenge and motivating the development of secure, robust, and privacy-preserving graph learning.
\item \textbf{Novel defense framework.} We develop \textsc{Veritas}, a poisoning-resilient locally private graph learning protocol built upon a trust-but-verify paradigm. \textsc{Veritas} introduces a four-stage defense pipeline that jointly enhances poisoning resilience, privacy preservation, and graph learning utility.
\item \textbf{Extensive experiments.} Experiments on real-world datasets across multiple LDP mechanisms and GNN architectures demonstrate that \textsc{Veritas} effectively defends against data poisoning attacks and significantly improves learning utility.
\end{itemize}

\section{Preliminaries}\label{S2}

This section first defines the problem ($\triangleright$ Sect.~\ref{2.1}), then provides the essential background on local differential privacy ($\triangleright$ Sect.~\ref{2.2}) and locally private graph learning ($\triangleright$ Sect.~\ref{2.3}).

\subsection{Problem Definition}\label{2.1}

Consider a graph $\mathcal{G} = (\mathcal{V}, \mathbf{X}, \mathbf{A})$, where 
$\mathcal{V} = \{v_1, v_2, \dots,$ $v_{|\mathcal{V}|}\}$ denotes the set of nodes. 
The feature matrix $\mathbf{X} \in \mathbb{R}^{|\mathcal{V}| \times d}$ associates each node $v\in\mathcal{V}$ 
with a $d$-dimensional feature vector $\mathbf{x}_v \in \mathbb{R}^d$, and the adjacency 
matrix $\mathbf{A} \in \{0,1\}^{|\mathcal{V}| \times |\mathcal{V}|}$ encodes the graph structure, 
where each row $\mathbf{a}_v \in \{0,1\}^{|\mathcal{V}|}$ is the adjacency vector of node $v$. In this work, we consider a comprehensive local privacy setting in which both node features 
and structural information are treated as sensitive. Each user $v$ locally holds the private pair 
$(\mathbf{x}_v, \mathbf{a}_v)$, capturing its attribute and relational information, respectively. 
A central server coordinates the learning process but is not assumed to be fully trustworthy; 
hence, neither $\mathbf{X}$ nor $\mathbf{A}$ is directly exposed to it. Instead, each user 
perturbs its local data via LDP mechanisms~\cite{sajadmanesh2021locally} prior to transmission, 
and the server performs graph learning on the collected noisy reports—forming a two-phase pipeline 
of \textit{user-side perturbation} and \textit{server-side learning}. However, the decentralized 
nature of locally private graph learning renders it inherently vulnerable to data poisoning attacks~\cite{he2025devil}, 
which can severely degrade the utility of downstream tasks 
(see Sect.~\ref{2.5} for the threat model). Our goal is therefore to develop a 
poisoning-resilient locally private graph learning framework that achieves accurate downstream 
graph learning while providing rigorous privacy guarantees.

\vspace{-0.5em}
\subsection{Local Differential Privacy}\label{2.2}

Local differential privacy (LDP)~\cite{kasiviswanathan2011can,he2025mitigating} has been extensively studied and widely deployed in decentralized data collection and analysis scenarios. In practice, it has been adopted by major industry systems, including Apple~\cite{patent1}, Google~\cite{erlingsson2014rappor}, 
and Microsoft~\cite{ding2017collecting}. In the LDP setting, a server interacts with multiple users, each holding sensitive data. Rather than sharing raw data with an untrusted server, each user locally perturbs its data 
via a randomized mechanism $\mathcal{M}$ before transmission. The server then performs data analysis and learning based solely on the perturbed data, thereby providing formal privacy guarantees. The formal definition of $\epsilon$-LDP is provided below:

\begin{mdframed}[
  linecolor=black,
  linewidth=0.6pt,
  skipabove=4pt,   
  skipbelow=4pt,   
  innertopmargin=2pt,    
  innerbottommargin=2pt  
]
\begin{definition}[$\epsilon\text{-LDP}$~\cite{kasiviswanathan2011can}]\label{def:1}
A local perturbation mechanism $\mathcal{M}$ satisfies $\epsilon\text{-LDP}$, where $\epsilon > 0$, if for any two inputs $x, x^\prime$ and any output $y \in \mathrm{Range}(\mathcal{M})$, it holds that
\begin{equation}
    \Pr[\mathcal{M}(x) = y] \le e^\epsilon \cdot\Pr[\mathcal{M}(x^\prime) = y],
	\end{equation}
\end{definition}
\end{mdframed}
where $\epsilon$ is the “privacy budget” that controls the privacy-utility trade-off: a smaller $\epsilon$ provides stronger privacy protection but typically leads to greater utility degradation. $\text{Range}(\mathcal{M})$ denotes the output space of $\mathcal{M}$.
In graph settings, protecting relational information requires extending LDP to link-level data. 
The resulting notion, known as $\epsilon$-link LDP~\cite{qin2017generating,
zhu2023blink,lin2022towards}, is defined as follows:

\begin{mdframed}[
  linecolor=black,
  linewidth=0.6pt,
  skipabove=4pt,   
  skipbelow=4pt,   
  innertopmargin=2pt,    
  innerbottommargin=2pt  
]
\begin{definition}[$\epsilon$-link LDP~\cite{qin2017generating}]\label{def:link-ldp}
A randomized mechanism $\mathcal{M}$ satisfies $\epsilon$-link LDP, if for any node $v$ abd for any two adjacency vectors $\mathbf{a}_v, \mathbf{a}_v^\prime \in \{0,1\}^{|\mathcal{V}|}$ differing in exactly one 
coordinate (i.e., $\|\mathbf{a}_v - \mathbf{a}_v^\prime\|_1 = 1$) and for any measurable set 
$\mathcal{S} \subseteq \mathrm{Range}(\mathcal{M})$, it holds that
\begin{equation}
    \Pr\bigl[\mathcal{M}(\mathbf{a}_v) \in \mathcal{S}\bigr] 
    \le e^{\epsilon} \cdot \Pr\bigl[\mathcal{M}(\mathbf{a}_v^\prime) \in \mathcal{S}\bigr].
\end{equation}
\end{definition}
\end{mdframed}

Intuitively, this definition ensures that an observer cannot reliably infer the presence or absence of any individual edge, thereby providing formal protection for relational information. Beyond these definitions, LDP satisfies two fundamental properties that are instrumental in our framework, as follows:

\begin{mdframed}[
  linecolor=black,
  linewidth=0.6pt,
  skipabove=4pt,   
  skipbelow=4pt,   
  innertopmargin=2pt,    
  innerbottommargin=2pt  
]
\begin{theorem}[Post-Processing~\cite{dwork2008differential}]\label{thm:post}
If $\mathcal{M}: \mathcal{X} \to \mathcal{Z}$ satisfy $\epsilon$-LDP, then for any (possibly randomized) function $\mathcal{F}: \mathcal{Z} \to \mathcal{Z}^\prime$, the composition 
$\mathcal{F} \circ \mathcal{M}$ also satisfies $\epsilon$-LDP.
\end{theorem}
\end{mdframed}
\begin{mdframed}[
  linecolor=black,
  linewidth=0.6pt,
  skipabove=4pt,   
  skipbelow=4pt,   
  innertopmargin=2pt,    
  innerbottommargin=2pt  
]
\begin{theorem}[Sequential Composition~\cite{dwork2008differential}]\label{thm:comp}
Let $\mathcal{M}_i: \mathcal{X} \to \mathcal{Z}_i$ satisfy $\epsilon_i$-LDP for 
$i \in \{1,2, \dots, n\}$. Applying all mechanisms to the same data yields a joint mechanism 
$\mathcal{M} = (\mathcal{M}_1,\mathcal{M}_2, \dots, \mathcal{M}_n)$ that satisfies 
$\left(\sum_{i=1}^{n} \epsilon_i\right)$-LDP.
\end{theorem}
\end{mdframed}

\subsection{Locally Private Graph Learning}\label{2.3}

We now formalize the locally private graph learning pipeline $\mathbb{P}_{\mathcal{M}, \mathcal{F}_\theta}$, where $\mathcal{M}$ denotes the 
user-side perturbation mechanism and $\mathcal{F}_\theta$ represents the 
server-side graph learning model parameterized by $\theta$, detailed as follows:

\subsubsection{User-Side Perturbation}\label{2.4.1}

Each node $v$ locally perturbs its private data before transmitting to the server.
Specifically, the node feature vector $\mathbf{x}_v \in \mathbb{R}^d$ and the binary adjacency 
vector $\mathbf{a}_v \in \{0,1\}^{|\mathcal{V}|}$ are processed via two perturbation mechanisms:
\begin{equation}
    \mathbf{x}_v^\prime \leftarrow \mathcal{M}_x(\mathbf{x}_v,\,\epsilon_x), \qquad
    \mathbf{a}_v^\prime \leftarrow \mathcal{M}_a(\mathbf{a}_v,\,\epsilon_a),
\end{equation}
where $\epsilon_x$ and $\epsilon_a$ denote the privacy budgets for node feature and link perturbation, 
respectively, satisfying $\epsilon_x + \epsilon_a = \epsilon$. We next detail the instantiation of these two mechanisms.
\begin{itemize}[leftmargin=*, itemindent=0em]
\renewcommand{\labelitemi}{$\diamond$}
    \item \textit{Node Feature Perturbation.} For a $d$-dimensional node feature vector $\mathbf{x}_v$, the perturbation proceeds in two steps:
\ding{172} a subset $\mathcal{S} \subseteq \{1,2,\dots,d\}$ of size $m$ is sampled uniformly;
\ding{173} each selected dimension $i \in \mathcal{S}$ is independently perturbed via an 
$(\epsilon_x/m)$-LDP mechanism, while the remaining $d - m$ dimensions are set to zero.
Beyond classical LDP mechanisms such as the Laplace mechanism~\cite{phan2017adaptive}, the Gaussian mechanism~\cite{balle2018improving}, 
 and the one-bit mechanism~\cite{ding2017collecting}, representative LDP mechanisms for node feature include the multi-bit mechanism~\cite{sajadmanesh2021locally}, 
the square wave mechanism~\cite{li2020estimating}, and the piecewise mechanism~\cite{wang2019collecting}.
As an illustrative example,
we describe the one-dimensional piecewise mechanism (\texttt{PM}). Given an input $x\in[\alpha,\beta]$, the perturbed output $x^\prime$ is sampled 
from the bounded interval $[-\mathcal{B}, \mathcal{B}]$, where $\mathcal{B} = \frac{e^{\epsilon_x/2}+1}{e^{\epsilon_x/2}-1}$, 
according to the following density:
\begin{equation}\label{eq2}
        \mathrm{Pr}[x' = c \mid x] =
        \begin{cases}
            p, & \text{if } c \in [l(x),\, r(x)], \\
            {p}/{e^{\epsilon_x}}, & \text{if } c \in [-\mathcal{B},\, l(x)) \cup (r(x),\, \mathcal{B}],
        \end{cases}
    \end{equation}
    where 
   \begin{equation}
    p = \frac{e^{\epsilon_x} - e^{{\epsilon_x}/2}}{2 e^{{\epsilon_x}/2} + 2}, l(x) = \frac{\mathcal{B} + 1}{2} x - \frac{\mathcal{B} - 1}{2},
    r(x) = l(x) + \mathcal{B} - 1.
    \end{equation}
\item \textit{Adjacency List Perturbation.} To protect the binary adjacency vector $\mathbf{a}_v$, we employ the randomized response (\texttt{RR}) mechanism~\cite{warner1965randomized}. Each entry $a_{v,u} \in \{0,1\}$ is independently flipped with 
probability inversely proportional to $e^{\epsilon_a}$:
\begin{equation}
a_{v,u}^\prime = 
\begin{cases}
a_{v,u}, & \text{with probability } e^{\epsilon_a}/({e^{\epsilon_a}+1}), \\
1-a_{v,u}, & \text{with probability } 1/(e^{\epsilon_a}+1).
\end{cases}
\label{eq4}
\end{equation}
\end{itemize}
By Theorem~\ref{thm:comp}, the combined mechanism 
$(\mathcal{M}_x, \mathcal{M}_a)$ satisfies $\epsilon$-LDP with 
$\epsilon = \epsilon_x + \epsilon_a$. Consequently, the server only has access to samples drawn from a perturbed 
data distribution induced by $\mathcal{M}$, rather than the true underlying data. The resulting perturbed data is then transmitted to the server, with no raw private data ever leaving the user's local device.

\subsubsection{Server-Side Learning}\label{2.4.2}

After receiving the perturbed reports $\{(\mathbf{x}_v^\prime, \mathbf{a}_v^\prime)\}_{v\in\mathcal{V}}$, the server reconstructs a noisy graph $\mathcal{G}^\prime = (\mathcal{V}, \mathbf{A}^\prime, \mathbf{X}^\prime)$, where $\mathbf{X}^\prime = [\mathbf{x}_1^\prime,\mathbf{x}_2^\prime, \dots, \mathbf{x}_{|\mathcal{V}|}^\prime]^\top$ and $\mathbf{A}^\prime$ is induced from $\{\mathbf{a}_v^\prime\}_{v\in\mathcal{V}}$. Based on $\mathcal{G}^\prime$, the server trains a $K$-layer graph neural network (GNN)~\cite{hamilton2017inductive,DBLP:conf/iclr/KipfW17} for downstream tasks such as node classification. Due to local perturbation, the true 
neighborhood $\mathcal{N}(v)$ is not directly observable; instead, the server operates 
on a perturbed neighborhood defined as $\widehat{\mathcal{N}}(v) = \{u \mid a_{v,u}^\prime = 1\}$. The node representations are then learned via iterative neighborhood aggregation. 
Specifically, at the $k$-th layer, the embedding of node $v$ is computed as:
\begin{align}
\mathbf{h}^{k}_{\widehat{\mathcal{N}}(v)} 
	&= \textsc{Aggregate}_k(\{\mathbf{h}_u^{k-1} \mid u \in \widehat{\mathcal{N}}(v)\}), \\
	\mathbf{h}^{k}_{v}     
	&= \textsc{Update}_k(\mathbf{h}^{k}_{\widehat{\mathcal{N}}(v)}),
\end{align}
where $\textsc{Aggregate}_k(\cdot)$ is a permutation-invariant function (e.g., mean or sum), 
and $\textsc{Update}_k(\cdot)$ is a learnable non-linear transformation. The input layer is 
initialized as $\mathbf{h}_v^0 = \mathbf{x}_v^\prime$. After $K$ layers, the learned representations $\mathbf{h}_v^K$ encode multi-hop structural 
information from the perturbed graph. The model is trained end-to-end on labeled nodes 
$\mathcal{V}_l$ and subsequently used to perform inference on unlabeled nodes 
$\mathcal{V}_u = \mathcal{V} \setminus \mathcal{V}_l$.

\section{Threat Model}\label{2.5}

In the locally differentially private graph learning setting, we consider two categories of security threats that may co-exist: \ding{172} \textit{privacy inference attacks} conducted by a semi-honest server ($\triangleright$ Sect.~\ref{3.1}), and \ding{173} \textit{data poisoning attacks} launched by a malicious external attacker ($\triangleright$ Sect.~\ref{3.2}).

\subsection{Privacy Inference Attacks by Semi-Honest Server}\label{3.1}

Under the semi-honest adversary model, the server faithfully adheres to the prescribed protocol $\mathbb{P}$ yet may simultaneously attempt to infer individual users' private information from their perturbed reports. Despite the rigorous probabilistic guarantees of LDP, a curious server can still exploit interactions with the trained GNN models to launch \textit{attribute inference attacks}~\cite{meng2023devil,wang2022group}, \textit{link inference attacks}~\cite{wu2022linkteller,meng2023devil}, or \textit{membership inference attacks}~\cite{zhang2022inference}. Such threats underscore the necessity of strict $\epsilon$-LDP requirements imposed on both features and links: even if the server aggregates all received data, the privacy of any individual user’s raw input remains formally protected.

\subsection{Data Poisoning Attacks by Malicious Attacker}\label{3.2}

Beyond inference-based privacy threats, the locally private graph learning protocol $\mathbb{P}$ is also vulnerable to data poisoning attacks~\cite{he2025devil}, as shown in Fig.~\ref{fig1}(b) and Alg.~\ref{alg:attack}. In such attacks, the attacker injects up to $n_{\text{max}}$ malicious nodes into the original graph $\mathcal{G}$ under a predefined attack budget $\Delta$, and strategically constructs adversarial feature and adjacency reports according to a set of attack strategies $\{\mathbb{S}_x, \mathbb{S}_{\text{target}}, \mathbb{S}_{\text{atk}}\}$. The objective is to compromise the utility of downstream graph learning tasks. Despite this significant threat, effective defenses against such poisoning attacks remain largely underexplored in prior work. We formally characterize the attacker along three dimensions:

\subsubsection{Adversary Capability}

The attacker is capable of injecting up to $n_{\text{max}}$ fake nodes, denoted by 
$\mathcal{V}_{\text{atk}}$ (with $1<|\mathcal{V}_{\text{atk}}| \ll |\mathcal{V}|$), into the original graph $\mathcal{G}= (\mathcal{V}, \mathbf{X}, \mathbf{A})$, and connecting them to a set of target victim nodes $\mathcal{V}_{t} \subseteq \mathcal{V}$. Each fake node $v_{\text{atk}} \in \mathcal{V}_{\text{atk}}$ is fully controlled by the attacker, who can arbitrarily construct its node feature vector $\tilde{\mathbf{x}}_{v_{\text{atk}}}^\dagger$ and adjacency vector $\tilde{\mathbf{a}}_{v_{\text{atk}}}^\dagger$ according to a set of attack strategies $\{\mathbb{S}_x, \mathbb{S}_{\text{target}}, \mathbb{S}_{\text{atk}}\}$. Specifically, $\mathbb{S}_x$ governs the generation of malicious features to manipulate the learned representations, 
$\mathbb{S}_{\text{target}}$ determines how fake nodes establish connections with target victim nodes 
to directly influence their local neighborhoods, and $\mathbb{S}_{\text{atk}}$ further optimizes the connectivity among fake nodes to amplify the overall attack effect. Importantly, since fake nodes masquerade as legitimate participants, the server cannot reliably distinguish them from genuine users based solely on the received (already perturbed) reports. By establishing links with victim nodes, the injected fake nodes effectively introduce adversarial neighbors into their aggregation neighborhoods, thereby influencing the downstream message-passing process.
\subsubsection{Adversary Knowledge}

We assume a realistic \textit{black-box} threat model: the attacker is assumed to have access only to the publicly available protocol $\mathbb{P}$, including the perturbation mechanisms $\mathcal{M}_x$ and $\mathcal{M}_a$, as well as the associated privacy budgets $\epsilon_x$ and $\epsilon_a$, since these parameters are typically transparent to all participants. Crucially, the attacker has no access to the raw node features, ground-truth labels, or the exact neighborhood information of genuine users. Furthermore, the internal states of the server-side GNN, such as the model architecture, parameters, and training dynamics, remain completely inaccessible to the attacker. This assumption is consistent with previous studies~\cite{he2025devil} on data poisoning attacks in the locally private graph learning protocol $\mathbb{P}$, and reflects a realistic scenario in which the attacker must create malicious inputs without access to structural knowledge or feedback from the target model.

\subsubsection{Adversary Objective}
The attacker aims to degrade the learning utility of the GNN model $\mathcal{F}_\theta$ trained on the poisoned graph. Let $\mathcal{G}^\dagger = (\mathcal{V} \cup \mathcal{V}_{\text{atk}}, \mathbf{A}^\dagger, \mathbf{X}^\dagger)$ denote the poisoned graph after node injection. The attacker's objective is to optimize the features and links of the injected nodes to maximize an impact metric $\mho(\cdot)$, while remaining within a stealth attack budget $\Delta$ to avoid detection. Formally, this objective is defined as:
\begin{equation}\label{eq:atk_formal}
\max_{\mathcal{V}_{\text{atk}}, \mathbf{X}^\dagger, \mathbf{A}^\dagger}  \mho(\mathcal{F}_\theta(\mathcal{G}^\dagger))\;\ 
\textit{s.t.}\;\mathrm{dist}\big((\mathbf{X}, \mathbf{A}), (\mathbf{X}^\dagger, \mathbf{A}^\dagger)\big) \le \Delta, 
\end{equation}
where the $\mathrm{dist}(\cdot, \cdot)$ function quantifies the magnitude of perturbations to enforce the stealthiness of the attack. In the context of node classification~\cite{DBLP:conf/iclr/KipfW17}, the metric $\mho(\cdot)$ usually corresponds to increasing the misclassification rate. 

\begin{algorithm}[t]
\centering
    \caption{Data Poisoning Attack Against Protocol $\mathbb{P}$}
    \label{alg:attack}
    \begin{algorithmic}[1]
        \renewcommand{\algorithmicrequire}{\textbf{Input:}}
        \renewcommand{\algorithmicensure}{\textbf{Output:}}
        \REQUIRE Graph $\mathcal{G} = (\mathcal{V}, \mathbf{X}, \mathbf{A})$, target node set $\mathcal{V}_{t}$, parameter $n_{\text{max}}$, protocol $\mathbb{P}$, attacker strategies $\{\mathbb{S}_x,  \mathbb{S}_{\text{target}},  \mathbb{S}_{\text{atk}}\}$.
        \ENSURE Fake node set $\mathcal{V}_{\text{atk}}$, reports $\{(\tilde{\mathbf{x}}_v^\dagger, \tilde{\mathbf{a}}_v^\dagger)\}_{v \in \mathcal{V}_{\text{atk}}}$.
        \STATE Initialize $\mathcal{V}_{\text{atk}} \leftarrow \emptyset$. \hfill $\rhd$ \textit{Create a malicious node set}
        \FOR{$i = 1$ \TO $n_{\text{max}}$}
            \STATE Create a fake node $v_{\text{atk}}$. 
            \STATE $\tilde{\mathbf{x}}_{v_{\text{atk}}}^\dagger \leftarrow \mathbb{S}_x(\mathbb{P})$. \hfill $\rhd$ \textit{Craft malicious node feature vector}
            \IF{$\mathcal{V}_{t} \neq \emptyset$}
            \STATE Select a subset of target nodes $\mathcal{T}_i \subseteq \mathcal{V}_{t}$.
            \STATE $\mathcal{V}_{t} \leftarrow \mathcal{V}_{t} \setminus \mathcal{T}_i$. 
        \hfill $\rhd$ \textit{Avoid repeated targeting}
            \ENDIF   
            \STATE $\tilde{\mathbf{a}}_{v_\text{atk}}^\dagger \leftarrow \mathbb{S}_{\text{target}}(\mathcal{T}_i, \mathcal{V})$. \hfill $\rhd$ \textit{Construct adjacency links}
            \STATE $\mathcal{V}_{\text{atk}} \leftarrow \mathcal{V}_{\text{atk}} \cup \{v_{\text{atk}}\}$.
        \ENDFOR
        \STATE $\{\tilde{\mathbf{a}}_v^\dagger\}_{v \in \mathcal{V}_{\text{atk}}} \leftarrow \mathbb{S}_{\text{atk}}(\mathcal{V}_{\text{atk}})$.  \hfill $\rhd$ \textit{Fake-fake coordination}
        \RETURN $\{(\tilde{\mathbf{x}}_v^\dagger, \tilde{\mathbf{a}}_v^\dagger)\}_{v \in \mathcal{V}_{\text{atk}}}$.
    \end{algorithmic}
\end{algorithm}

\begin{figure*}[t]
  \centering
\includegraphics[scale=0.25]{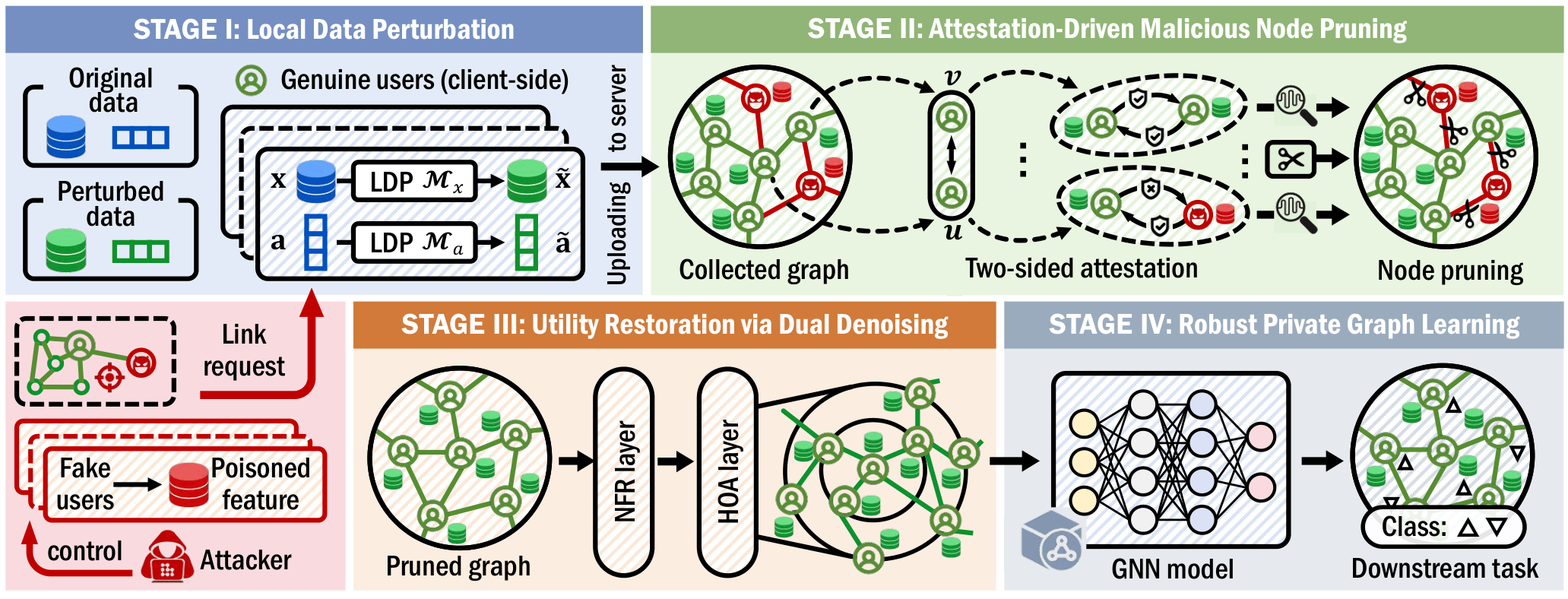}
  \caption{\textbf{Overview of \textsc{Veritas}.}
\textsc{Veritas} comprises four synergistic stages: \ding{172} local data perturbation for private data collection, \ding{173} attestation-driven malicious node pruning for adversarial participant removal, \ding{174} utility restoration via dual denoising, and \ding{175} robust private graph learning for downstream inference.}
  \label{fig2}
  \vspace{-1em}
\end{figure*}

\section{Methodology}\label{S4}

In this section, we introduce \textsc{Veritas}, a poisoning-resilient protocol designed for locally private graph learning. Built upon a \emph{trust-but-verify} paradigm, \textsc{Veritas} allows all users to participate in the learning process by first \emph{trusting} their local reports while subsequently \emph{verifying} their consistency and credibility on the server side. As illustrated in Fig.~\ref{fig2} and Alg.~\ref{alg3}, \textsc{Veritas} integrates four synergistic stages: \ding{172} \textit{local data perturbation} for private data collection, \ding{173} \textit{attestation-driven malicious node pruning} for verifying structural consistency and eliminating adversarial influence, \ding{174} \textit{utility restoration via dual denoising} to mitigate perturbation-induced noise, and \ding{175} \textit{robust private graph learning} for downstream tasks. By seamlessly coupling these components, \textsc{Veritas} establishes an end-to-end defense pipeline that embodies the principle of ``\textit{trust first, verify later}'', effectively safeguarding locally private graph learning against poisoning attacks while maintaining rigorous formal privacy guarantees. The subsequent sections ($\triangleright$ \ref{sec4.1}-\ref{sec4.4}) detail each stage, with a theoretical analysis provided in Sect.~\ref{sec4.5}.

\vspace{-1em}
\subsection{Local Data Perturbation}\label{sec4.1}
\vspace{-0.2em}

This stage performs LDP perturbation on both node features and graph structure before any data transmission. Unlike previous work~\cite{hegoing,li2024privacy,he2025devil} focusing solely on feature protection, we jointly privatize node features and graph structure, ensuring comprehensive privacy guarantees for graph-structured data.  

\subsubsection{Node Feature Perturbation}

Given a $d$-dimensional node feature vector $\mathbf{x}_v\in[\alpha,\beta]^d$, direct access to original node features is prohibited under LDP constraints. To address this, each user $v$ perturbs $\mathbf{x}_v$ locally using an LDP mechanism $\mathcal{M}_x$ to obtain a privatized representation $\tilde{\mathbf{x}}_v$. We adopt a unified node feature LDP perturbation pipeline, which generalizes several representative mechanisms, including \texttt{PM}~\cite{wang2019collecting}, \texttt{MB}~\cite{sajadmanesh2021locally}, and \texttt{SW}~\cite{li2020estimating}. The pipeline is summarized in Alg.~\ref{alg1}. Specifically, a perturbation budget $\epsilon_x$ is allocated to feature protection, and a subset $\mathcal{S} \subseteq [d]$ of size $m$ is uniformly sampled (\textit{Line 1}). Each selected dimension $i \in \mathcal{S}$ is independently perturbed using an $(\epsilon_x/m)$-LDP mechanism, while the remaining dimensions are set to zero (\textit{Lines 2-7}). This sampling strategy effectively reduces per-dimension noise, thereby improving utility under a fixed privacy budget. Optionally, a post-processing step $\textsc{Rect}(\cdot)$ (\textit{e.g.}, bias correction) can be applied to ensure unbiased estimation, \textit{i.e.}, $\mathbb{E}[\tilde{\mathbf{x}}_v] \approx \mathbf{x}_v$ (\textit{Line 8}).

\subsubsection{Graph Structure Perturbation}

In addition to protecting node features $\mathbf{x}_v$, we also provide privacy protection for the graph structure. Instead of adopting the conventional \texttt{RR}-based perturbation mechanism introduced in Sect.~\ref{2.3}, we propose a novel data structure, termed the \textit{verification list}, to jointly encode relational information and peer trust, thus enhancing robustness against poisoning attacks. Specifically, each node $v$ maintains a verification list $\breve{\mathbf{a}}_v \in \{0,1,\ldots,\kappa\}^{|\mathcal{V}|}$, where each entry $\breve{a}_{v,u}$ represents node $v$’s assessment of its relationship with node $u$, defined as follows:
\begin{equation}
    \breve{a}_{v,u} =
    \begin{cases}
        0 & \text{if } u \notin \mathcal{N}(v) 
            \quad \text{(non-neighbor)}, \\
        1, 2, \ldots, \kappa & \text{if } u \in \mathcal{N}(v) 
            \quad \text{(trust level)}.
    \end{cases}
    \label{eq:verlist}
\end{equation}
The level $0$ serves as a dedicated indicator of structural absence and is semantically distinct from the trust levels $\{1,2,\ldots,\kappa\}$ assigned to the confirmed neighbors. To ensure $\epsilon_a$-link LDP (Definition~\ref{def:link-ldp}), each entry $r_{v,u}$ is independently perturbed
using the $K$-ary randomized response (\texttt{KRR}) mechanism~\cite{kairouz2016discrete}:
\begin{equation}
    \mathrm{Pr}\!\left(\tilde{a}_{v,u} = j \mid \breve{a}_{v,u} = i\right) =
    \begin{cases}
        \dfrac{e^{\epsilon_a}}{e^{\epsilon_a} + \kappa} 
            & \text{if } j = i, \\[6pt]
        \dfrac{1}{e^{\epsilon_a} + \kappa} 
            & \text{if } j \neq i,
    \end{cases}
    \label{eq:krr}
\end{equation}

where $\epsilon_a$ denotes the privacy budget allocated to the perturbation of the verification list. This mechanism satisfies $\epsilon_r$-link local differential privacy, as stated in the following proposition:

\begin{mdframed}[
  linecolor=black,
  linewidth=0.6pt,
  skipabove=4pt,   
  skipbelow=4pt,   
  innertopmargin=2pt,    
  innerbottommargin=2pt  
]
\begin{proposition}
    \label{prop:krr_ldp}
        The \texttt{KRR} mechanism defined in Eq.~(\ref{eq:krr}) satisfies $\epsilon_a$-link LDP as defined in Definition~\ref{def:link-ldp}.

\end{proposition}
\begin{proof}
   For any two inputs $i, i' \in \{0,1,\ldots,\kappa\}$ and any output $j$, the privacy ratio is bounded as
    \begin{equation}
    \frac{P(\tilde{a}=j \mid \breve{a}=i)}{P(\tilde{a}=j \mid \breve{a}=i')} \leq e^{\epsilon_a},
    \end{equation}
    which follows directly from the mechanism definition and the standard analysis of \texttt{KRR} mechanism~\cite{kairouz2016discrete}.
\end{proof}
\end{mdframed}

Finally, the perturbed pair $\{(\tilde{\mathbf{x}}_v, \tilde{\mathbf{r}}_v)\}_{v\in\mathcal{V}}$
is transmitted to the server. By the composition theorem of LDP
(Theorem~\ref{thm:comp}), since $\tilde{\mathbf{x}}_v$ and $\tilde{\mathbf{a}}_v$
pertain to disjoint attributes, the overall perturbation satisfies
$\epsilon$-LDP with $\epsilon = \epsilon_x + \epsilon_a$.

\subsection{Attestation-Driven Malicious Node Pruning}
\label{sec4.2}

\begin{algorithm}[t]
    \caption{Node Feature LDP Perturbation Mechanism}
    \label{alg1}
    \renewcommand{\algorithmicrequire}{\textbf{Input:}}
     \renewcommand{\algorithmicensure}{\textbf{Output:}}
    \begin{algorithmic}[1]
        \REQUIRE Node feature $\mathbf{x}_v\in[\alpha,\beta]^d$, privacy budget $\epsilon_x>0$.
        \ENSURE Perturbed node feature $\tilde{\mathbf{x}}_v\in[-\mathcal{B},\mathcal{B}]^d$.
        \STATE  $m\leftarrow \max (1,\min ( d$, $\lfloor\delta\cdot \epsilon \rfloor)$. \hfill $\rhd$ \textit{perturbation size}
         \STATE Let $\mathbf{x}^\prime_v\leftarrow<0,0,\cdots,0>$.
        \STATE Let $\mathcal{S}\subset [d]$ denote a subset of $m$ distinct dimensions selected uniformly at random without replacement \\from the index set $\{1, 2, \ldots, d\}$.
        \FOR{\emph{each sampled dimension} $i\in\mathcal{S}$}
    \STATE Feed $\mathbf{x}_v[i]$ and $\epsilon/m$ as input to one-dimensional \\LDP mechanisms (\textit{e.g.}, Eq.~(\ref{eq2})), and obtain ${x}^\prime_{\text{tmp}}$.
    \STATE ${\mathbf{x}}^\prime_v[i]\leftarrow {x}^\prime_{\text{tmp}}$.\hfill $\rhd$ \textit{LDP perturbation}
    \ENDFOR
    \STATE   $\tilde{\mathbf{x}}_v \gets \textsc{Rect}(\mathbf{x}^\prime_v,\epsilon,m,d)$.\hfill $\rhd$ \textit{unbiased rectification}
   \RETURN Perturbed node feature $\tilde{\mathbf{x}}_v\in[-\mathcal{B},\mathcal{B}]^d$.
    \end{algorithmic}
\end{algorithm}

Upon receiving ${(\tilde{\mathbf{x}}_v, \tilde{\mathbf{a}}_v)}_{v\in\mathcal{V}}$, the server proceeds to assess the credibility of each reported relationship through \textit{cross-node attestation signals}. The central insight driving this stage lies in the \emph{incentive asymmetry} inherent in graph poisoning. A malicious node $v_\text{atk}$ tends to assign high trust levels to its target neighbors to maximize its adversarial impact, whereas a benign node evaluates its peers based on local observations and naturally assigns low trust to structurally anomalous neighbors.
This behavioral divergence induces a systematic asymmetry in bilateral verification entries, which can be interpreted as inconsistencies in mutual attestation. Importantly, this signal persists even under LDP noise and remains identifiable on the server. \textsc{Veritas} takes advantage of this property through a three-step procedure consisting of \ding{172} \textit{bilateral signal extraction}, \ding{173} \textit{edge confidence scoring}, and \ding{174} \textit{suspicion-based pruning}.

\subsubsection{Bilateral Signal Extraction}

For each candidate pair $(u,$ $v)$, the server observes two independent perturbed entries, $\tilde{a}_{v,u}$ (node $v$’s assessment of node $u$) and $\tilde{a}_{u,v}$ (node $u$’s assessment of node $v$), forming a bilateral attestation signal. Based on this observation, three complementary signals are extracted:
\begin{itemize}[leftmargin=*, itemindent=0em]
\renewcommand{\labelitemi}{$\diamond$}
    \item \textit{Signal 1 (Structural existence).} Both parties must mutually acknowledge the relationship for an edge to be considered structurally valid: $s^{\mathrm{exist}}_{uv} =
    \mathbf{1}\!\left[\tilde{a}_{v,u} \neq 0\right] \cdot
    \mathbf{1}\!\left[\tilde{a}_{u,v} \neq 0\right]$. An edge flagged by either party as non-existent ($\tilde{a}=0$) is structurally
unconfirmed and receives zero confidence.
\item \textit{Signal 2 (Trust consistency).} Among mutually acknowledged edges, the {effective trust} is determined by the lower of the two reported levels: $s^{\mathrm{cons}}_{uv} = \min\!\left(\tilde{a}_{v,u},\, \tilde{a}_{u,v}\right)$. The $\min$ operator is a deliberate design choice grounded in adversarial reasoning: a malicious node $v_\text{atk}$ will strategically report the maximum trust level $\kappa$ toward all neighbors to maintain influence, so the minimum operator effectively neutralizes this manipulation; the overall trust is anchored by the honest party's conservative assessment.
\item \textit{Signal 3 (Endorsement asymmetry).} The absolute discrepancy between the two levels reported quantifies the degree
of bilateral inconsistency: $s^{\mathrm{asym}}_{uv} = \left|\tilde{a}_{u,v} - \tilde{a}_{v,u}\right|$. A high asymmetry score indicates that one party's assessment deviates
substantially from the other's, which is a characteristic signature of
adversarial behavior.
\end{itemize}

\subsubsection{Edge Confidence Scoring}

The three signals are consolidated into a unified edge confidence score
$c_{uv} \in [0, 1]$:
\begin{equation}
    c_{uv} =
    s^{\mathrm{exist}}_{uv}
    \cdot \frac{s^{\mathrm{cons}}_{uv}}{\kappa}
    \cdot \exp\!\left(-\lambda \cdot s^{\mathrm{asym}}_{uv}\right),
    \label{eq:confidence}
\end{equation}
where $\lambda > 0$ is a hyperparameter controlling the penalty for asymmetry.
The three multiplicative terms operate at distinct levels of granularity: the
existence term acts as a hard gate that suppresses unconfirmed edges entirely;
the consistency term provides a soft, graded measure of mutual trust; and the
asymmetry term exponentially penalizes one-sided endorsements. Together,
Eq.~(\ref{eq:confidence}) assigns high confidence only to edges for which both
parties report comparable and positive trust levels, while systematically
discounting edges that exhibit the hallmarks of adversarial manipulation.

\subsubsection{Suspicion Scoring and Node Pruning}

Edge-level confidence scores are aggregated into a 
{node-level suspicion score} $\phi_v$, which captures the 
systematic directional bias in $v$'s bilateral assessments:
\begin{equation}
    \phi_v =
    \frac{1}{|\hat{\mathcal{N}}_v|}
    \sum_{u \in \hat{\mathcal{N}}_v}
    \left(\tilde{a}_{u,v} - \tilde{a}_{v,u}\right),
    \label{eq:suspicion}
\end{equation}
where $\hat{\mathcal{N}}_v = \{u : \tilde{a}_{v,u} > 0 \vee 
\tilde{a}_{u,v} > 0\}$ denotes the set of nodes claiming any 
relationship with $v$. A positive $\phi_v$ indicates that $v$ 
consistently assigns higher trust to others than others assign 
to it, precisely the behavioral fingerprint of a malicious 
node maximizing its structural influence while being rated 
poorly by honest peers.

\begin{mdframed}[linecolor=black,linewidth=0.6pt,
  skipabove=4pt,skipbelow=4pt,
  innertopmargin=2pt,innerbottommargin=2pt]
\begin{proposition}[Suspicion Score Separability]
    \label{prop:sep}
    Let $v_m \in \mathcal{M}$ be a fake node that reports 
    maximum trust $\breve{a}_{v_m, u} = \kappa$ toward all 
    neighbors, and let each honest neighbor $u \notin 
    \mathcal{M}$ assign trust level $\breve{a}_{u, v_m} = 1$ 
    due to anomalous local observations. Then, after 
    $(\kappa{+}1)$-\texttt{KRR} perturbation with budget 
    $\epsilon_a$, the expected suspicion score satisfies:
    \begin{equation}
        \mathbb{E}[\phi_{v_m}]
        = (2p_a - 1)(\kappa - 1) > 0,
        \label{eq:sep}
    \end{equation}
    where $p_a = e^{\epsilon_a}/(e^{\epsilon_a} + \kappa)$ is 
    the truthful reporting probability of \texttt{KRR}. For 
    an honest node $v_h \notin \mathcal{M}$ with symmetric 
    bilateral assessments, $\mathbb{E}[\phi_{v_h}] = 0$.
\end{proposition}
\end{mdframed}

Proposition~\ref{prop:sep} establishes that the expected 
suspicion gap between malicious and honest nodes is 
$(2p_a-1)(\kappa-1)$, which grows with both $\epsilon_a$ 
(privacy budget) and $\kappa$ (granularity of trust levels), 
providing a clear principle for hyperparameter selection. Finally, nodes whose suspicion score exceeds threshold 
$\gamma$ are flagged and removed:
\begin{equation}
    \hat{\mathcal{M}} = \{v \in \mathcal{V} : \phi_v > 
    \gamma\}, \qquad
    \hat{\mathcal{V}} = \mathcal{V} \setminus \hat{\mathcal{M}}.
    \label{eq:pruning}
\end{equation}
The pruned graph $\hat{\mathcal{G}} = (\hat{\mathcal{V}},\, 
\hat{\mathcal{E}})$, where $\hat{\mathcal{E}} = \{(u,v) \in 
\hat{\mathcal{V}}^2 : c_{uv} > 0\}$, retains only nodes and 
edges with positive bilateral support and is passed to 
Stage~\ding{174}.

\subsection{Utility Restoration via Dual Denoising}\label{sec4.3}

After malicious node pruning, the retained graph 
$\hat{\mathcal{G}}$ still suffers from utility degradation 
induced by LDP perturbation. To recover signal quality, we 
adopt the \textit{Node Feature Regularization} (\texttt{NFR}) 
and \textit{High-Order Aggregator} (\texttt{HOA}) components 
from~\cite{hegoing}: \texttt{NFR} reduces 
effective feature dimensionality via $L_1$-regularization to 
lower per-dimension data estimation error, while \texttt{HOA} 
expands the effective neighborhood through personalized 
multi-hop aggregation to reduce noise variance. Together, \texttt{NFR} and \texttt{HOA} produce a denoised 
feature matrix $\hat{\mathbf{X}}$ and refined adjacency 
structure $\hat{\mathbf{A}}$, which are subsequently fed into 
the robust private graph learning stage.

\begin{algorithm}[t]
\centering
\caption{\textsc{Veritas}}
\label{alg3}
\begin{algorithmic}[1]
    \renewcommand{\algorithmicrequire}{\textbf{Input:}}
    \renewcommand{\algorithmicensure}{\textbf{Output:}}
    \REQUIRE Graph $\mathcal{G} = (\mathcal{V}, \mathbf{X}, 
    \mathbf{A})$, privacy budgets $\epsilon_x$, $\epsilon_a$, 
    trust levels $\kappa$, thresholds $\gamma$, $\lambda$, 
    hops $K$, labeled nodes $\mathcal{V}_l$.
    \ENSURE Trained model $\mathcal{F}_\theta$, predictions 
    on $\mathcal{V}_u$.

    \STATE \textcolor{cyan}{\text{// Stage \ding{172}: Local Data 
    Perturbation}}
    \FOR{each user $v \in \mathcal{V}$}
        \STATE $\tilde{\mathbf{x}}_v \leftarrow 
        \mathcal{M}_x(\mathbf{x}_v, \epsilon_x)$ via 
        Alg.~\ref{alg1}. \hfill $\rhd$ \textit{feature 
        perturbation}
        \STATE Construct $\breve{\mathbf{a}}_v 
        \in \{0,1,\ldots,\kappa\}^{|\mathcal{V}|}$ via Eq.~(\ref{eq:verlist}).
        \STATE $\tilde{\mathbf{r}}_v \leftarrow 
        \texttt{KRR}(\breve{\mathbf{a}}_v, \epsilon_a)$. \hfill $\rhd$ \textit{structure 
        perturbation}
        \STATE Upload $(\tilde{\mathbf{x}}_v, 
        \tilde{\mathbf{r}}_v)$ to server.
    \ENDFOR

    \STATE \textcolor{cyan}{\text{// Stage \ding{173}: Attestation-Driven 
    Malicious Node Pruning}}
    \FOR{each candidate pair $(u, v)$}
        \STATE Extract signals $s^{\mathrm{exist}}_{uv}$, 
        $s^{\mathrm{cons}}_{uv}$, $s^{\mathrm{asym}}_{uv}$ 
        from $(\tilde{a}_{v,u}, \tilde{a}_{u,v})$.
        \STATE Compute edge confidence $c_{uv}$ via 
        Eq.~(\ref{eq:confidence}).
    \ENDFOR
    \FOR{each node $v \in \mathcal{V}$}
        \STATE Compute suspicion score $\phi_v$ via 
        Eq.~(\ref{eq:suspicion}).
    \ENDFOR
    \STATE $\hat{\mathcal{M}} \leftarrow \{v \in \mathcal{V} : 
    \phi_v > \gamma\}$; $\hat{\mathcal{V}} \leftarrow 
    \mathcal{V} \setminus \hat{\mathcal{M}}$. \hfill 
    $\rhd$ \textit{prune nodes}
    \STATE Construct pruned graph $\hat{\mathcal{G}} = 
    (\hat{\mathcal{V}}, \hat{\mathcal{E}})$.

    \STATE \textcolor{cyan}{\text{// Stage \ding{174}: Utility Restoration 
    via Dual Denoising}}
    \STATE $\hat{\mathbf{X}} \leftarrow 
    \texttt{NFR}(\{\tilde{\mathbf{x}}_v\}_{v \in 
    \hat{\mathcal{V}}}, \epsilon_x)$. \hfill $\rhd$ 
    \textit{node feature regularization}
    \STATE $\hat{\mathbf{A}} \leftarrow 
    \texttt{HOA}(\hat{\mathcal{G}}, K)$. \hfill $\rhd$ 
    \textit{higher-order aggregator}

    \STATE \textcolor{cyan}{\text{// Stage \ding{175}: Robust Private Graph 
    Learning}}
    \STATE Train $\mathcal{F}_\theta$ on 
    $(\hat{\mathbf{X}}, \hat{\mathbf{A}})$ using labeled 
    nodes $\mathcal{V}_l \cap \hat{\mathcal{V}}$.
    \STATE Perform inference on $\mathcal{V}_u = 
    \hat{\mathcal{V}} \setminus \mathcal{V}_l$.
    \RETURN Trained model $\mathcal{F}_\theta$ and 
    predictions on $\mathcal{V}_u$.
\end{algorithmic}
\end{algorithm}

\vspace{-0.5em}
\subsection{Robust Private Graph Learning}\label{sec4.4}
Upon receiving the denoised graph $(\hat{\mathbf{X}}, 
\hat{\mathbf{A}})$, the server trains a GNN 
$\mathcal{F}_\theta$ for downstream tasks such as node 
classification on the labeled node set $\mathcal{V}_l$, and 
subsequently performs inference on the unlabeled set 
$\mathcal{V}_u = \hat{\mathcal{V}} \setminus \mathcal{V}_l$. 
Since malicious nodes have been pruned and perturbation-induced noise has been mitigated in the preceding stages, the 
GNN operates on a substantially cleaner graph, yielding more 
reliable node representations and improved downstream 
utility. The complete \textsc{Veritas} protocol is summarized 
in Algorithm~\ref{alg3}.

\vspace{-0.5em}
\subsection{TheoreticalAnalysis}\label{sec4.5}
\subsubsection{Privacy Analysis}
On the user side, perturbing node features and verification lists under $\epsilon_x$-LDP and $\epsilon_a$-link LDP, respectively, guarantees $(\epsilon_x+\epsilon_a)$-LDP by sequential composition (Theorem~\ref{thm:comp}). On the server side, all 
subsequent operations are performed solely on the collected reports, and by post-processing invariance 
(Theorem~\ref{thm:post}), no additional privacy budget is 
consumed. Thus, \textsc{Veritas} satisfies $\epsilon$-LDP with $\epsilon = \epsilon_x + \epsilon_a$, providing an end-to-end formal privacy guarantee.

\subsubsection{Complexity Analysis}
The attestation-driven pruning stage (Stage~\ding{173}) incurs a computational complexity of $\mathcal{O}(|\mathcal{V}| + |\mathcal{E}|)$ for bilateral signal extraction and suspicion score computation. The dual denoising stage (Stage~\ding{174}) requires $\mathcal{O}(K|\mathcal{E}|d)$ and $\mathcal{O}(|\mathcal{V}|d)$ for \texttt{HOA} and \texttt{NFR}, respectively, where $K$ denotes the number of propagation steps and $d$ is the feature dimension. Consequently, both stages exhibit near-linear complexity with respect to the graph size, making \textsc{Veritas} scalable to large-scale graph learning scenarios.

\begin{figure*}[t]
    \centering
     \begin{minipage}{\textwidth}
        \centering
        \subfigure[\normalsize Cora]{\includegraphics[width=0.47\textwidth]{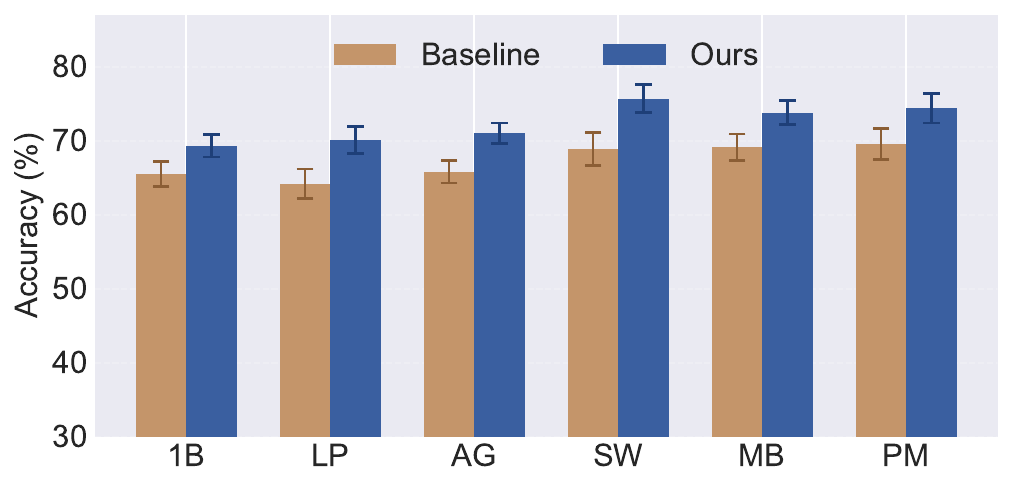}}
        \subfigure[\normalsize Citeseer]{\includegraphics[width=0.47\textwidth]{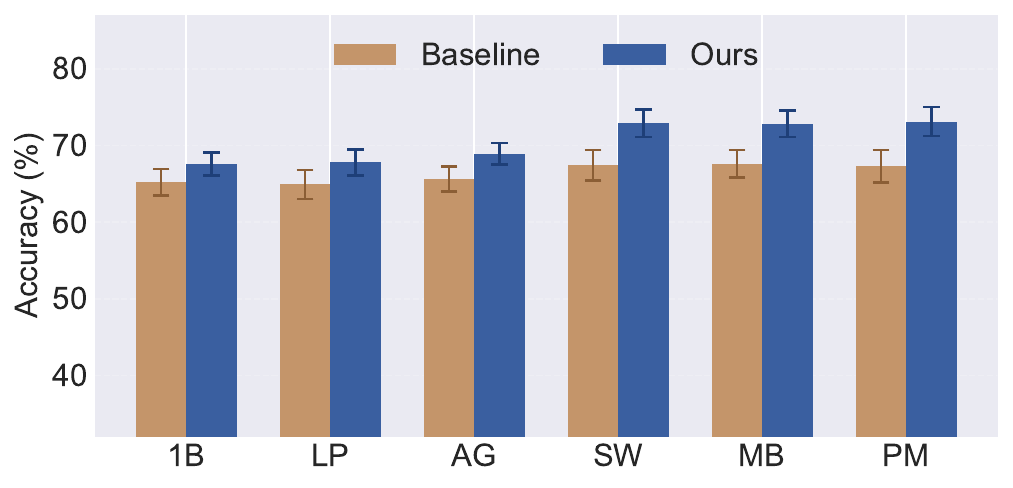}}
     \end{minipage}\\ 
     \begin{minipage}{\textwidth}
        \centering
        \subfigure[\normalsize LastFM]{\includegraphics[width=0.47\textwidth]{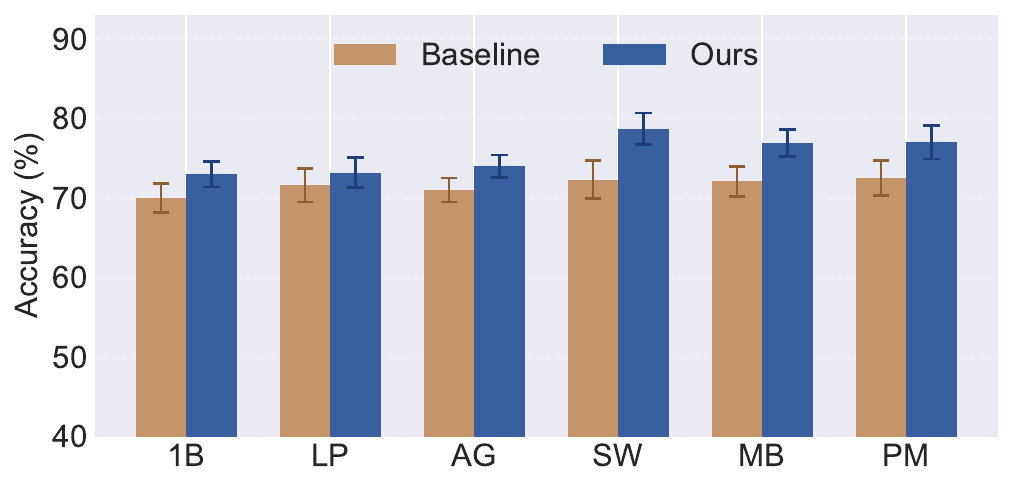}}
        \subfigure[\normalsize Twitch]{\includegraphics[width=0.47\textwidth]{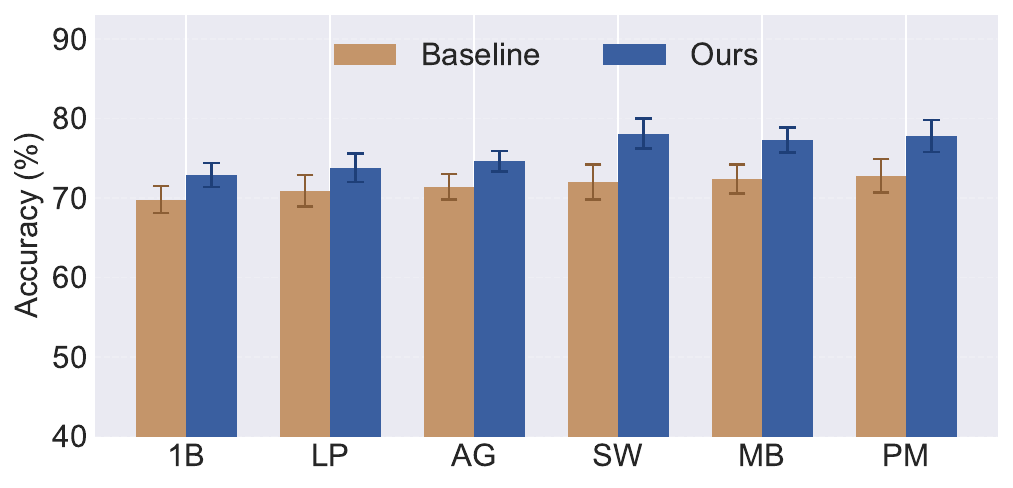}}
     \end{minipage}\\ 
   \caption{Performance comparison between our proposed \textsc{Veritas} (Ours) and attack baseline on the node classification task.}
\vspace{-0.5em}
    \label{fig4}
\end{figure*}
\section{Experiments}\label{S5}

In this section, we empirically evaluate the effectiveness of \textsc{Veritas}\footnote{The code will be made publicly 
available upon acceptance.}. Sect.~\ref{sec5.1} outlines the experimental settings, while Sect.~\ref{sec5.2} reports the results and provides detailed discussions.

\begin{table}
	\centering
	\small
 \caption{Statistics of graph datasets}
	\setlength{\tabcolsep}{1.5mm}\begin{tabular}{l|l|rrrr}
    \toprule
    \textbf{Type} & \textbf{Dataset} & \textbf{\#Nodes} & \textbf{\#Edges} & \textbf{\#Features} & \textbf{\#Classes} \\
    \midrule
    \multirow{2}{*}{\makecell[l]{Citation\\Network}} 
    & Cora & 2,708 & 5,278 & 1,433 & 7  \\
     & Citeseer & 3,327 & 4,552 & 3,703  & 6  \\
     \midrule
    \multirow{2}{*}{\makecell[l]{Social\\Network}} 
    & LastFM & 7,624 & 27,806 & 7,842 & 18 \\
    & Twitch & 4,648 & 61,706 & 128 & 2 \\
    \bottomrule
    \end{tabular}
 \label{tab1}
 \vspace{-1em}
\end{table}
\subsection{Experimental Settings}\label{sec5.1}
\subsubsection{Datasets}
We conduct experiments on four benchmark graph datasets 
spanning two categories: two citation networks, 
\textit{Cora}~\cite{yang2016revisiting} and \textit{Citeseer}~\cite{yang2016revisiting}, where 
nodes represent papers and edges denote citation 
relationships; and two social networks, 
\textit{LastFM}~\cite{rozemberczki2020characteristic} and 
\textit{Twitch}~\cite{rozemberczki2020characteristic}, where 
nodes represent users and edges encode social links. 
These datasets have been broadly adopted for evaluating graph 
learning methods under privacy 
constraints~\cite{lin2022towards,hegoing,he2025devil}. 
Table~\ref{tab1} summarizes their statistics.

\subsubsection{GNN Models}
We consider three representative GNN 
backbones: GCN~\cite{DBLP:conf/iclr/KipfW17}, 
GraphSAGE~\cite{hamilton2017inductive}, and 
GAT~\cite{DBLP:conf/iclr/VelickovicCCRLB18}. Unless otherwise specified, 
each model is implemented with two graph convolutional layers 
with hidden dimension 64. We implement
the GNN models in PyTorch using PyTorch-Geometric (PyG)\footnote{https://www.pyg.org/}. The GCN serves as the default 
backbone.

\subsubsection{LDP Mechanisms}
We consider six LDP mechanisms for node feature protection, 
covering both classical and recent approaches. The classical 
baselines include the 1-bit mechanism 
(\texttt{1B})~\cite{ding2017collecting}, the Laplace 
mechanism (\texttt{LP})~\cite{phan2017adaptive}, and the 
Analytic Gaussian mechanism 
(\texttt{AG})~\cite{balle2018improving}. We further include 
three advanced mechanisms that have demonstrated strong 
utility-privacy trade-offs in graph settings: the piecewise 
mechanism (\texttt{PM})~\cite{wang2019collecting}, the 
multi-bit mechanism (\texttt{MB})~\cite{sajadmanesh2021locally}, 
and the square wave mechanism 
(\texttt{SW})~\cite{li2020estimating}. \texttt{PM} is used as the default mechanism. For graph structure 
protection, the verification list is perturbed via 
\texttt{KRR} as described in Sect.~\ref{sec4.1}. 
\subsubsection{Parameter Settings}
 The feature privacy budget $\epsilon_x$ 
is varied over $\{0.01, 0.1, 1.0,2,0,3,0\}$ and the structural privacy budget $\epsilon_a$ over $\{1.0,3.0, 5.0,7.0, 9.0\}$, with defaults 
$\epsilon_x = 0.1$ and $\epsilon_a = 5.0$. All datasets are split into 50\% training, 25\% validation, 
and 25\% test sets. The granularity is 
$\kappa = 5$. The injection rate $r^\star$ is varied over 
$\{0.001, 0.005, 0.01, 0.02, 0.04\}$ with default $r^\star 
= 0.01$. The threshold $\gamma$ and $\lambda$ are tuned in the validation set. Utility 
restoration follows~\cite{hegoing}. 

\subsubsection{Evaluation Metrics}
We evaluate task performance on the node 
classification task~\cite{DBLP:conf/iclr/KipfW17}, measured by test accuracy (\%). 
All results are averaged over 10 independent runs with 
95\% confidence intervals. Experiments are 
run on a server equipped with Ubuntu 20.04 LTS, two 
Intel\textsuperscript{\textregistered} 
Xeon\textsuperscript{\textregistered} Gold 6348 CPUs, 
100GB RAM, and an NVIDIA\textsuperscript{\textregistered} 
A800 80GB GPU.

\begin{table}[t]
\vspace{-1em}
	\centering
	\small
 \caption{Comparison of classification accuracy between our
defense and attack baseline across different GNN models.}
	\setlength{\tabcolsep}{1.5mm}\begin{tabular}{c|c|cccc}
    \toprule
    \textbf{GNN} &  \textbf{Method} & \textbf{Cora} & \textbf{Citeseer} & \textbf{LastFM} & \textbf{Twitch}  \\
    \midrule
    \multirow{2}{*}{\makecell[l]{GraphSAGE}} & Baseline
    & 69.2 & 67.6 & 72.4 & 72.7   \\
     & \textbf{Ours} & \textbf{75.7} & \textbf{72.9} & \textbf{76.3} & \textbf{76.9}    \\
     \midrule
    \multirow{2}{*}{\makecell[l]{GAT}} 
    & Baseline & 67.4 & 65.1 & 69.6 & 70.1  \\
    & \textbf{Ours}  & \textbf{71.8} & \textbf{70.3} & \textbf{75.2} & \textbf{75.4}  \\
    \bottomrule
    \end{tabular}
 \label{tab2}
 \vspace{-1.5em}
\end{table}

\subsection{Experimental Results}\label{sec5.2}

\subsubsection{Effectiveness}
We evaluate the effectiveness of \textsc{Veritas} against data poisoning attacks on four benchmark datasets under the default privacy budget setting, using the poisoned model without defense as the baseline. As shown in Fig.~\ref{fig4}, \textsc{Veritas} consistently improves node classification accuracy across all datasets and LDP mechanisms, demonstrating strong resilience to poisoning attacks under diverse perturbation schemes. Table~\ref{tab2} further verifies the robustness of \textsc{Veritas} across different GNN architectures. Under both GraphSAGE and GAT backbones, our framework substantially restores the utility degraded by poisoning, indicating that its effectiveness is largely independent of the underlying GNN design. This is mainly because \textsc{Veritas} operates at the data and participant filtering level rather than modifying model-specific architectures, allowing it to generalize across different GNN encoders. Overall, these results demonstrate the strong generalizability of \textsc{Veritas} across both privacy mechanisms and GNN models.

\subsubsection{Ablation Studies}
To quantify the contribution of each core defense stage, we 
compare four configurations: \textbf{Baseline} (poisoned 
without any defense), \textbf{w/o II} (full 
\textsc{Veritas} without attestation-driven malicious node 
pruning), \textbf{w/o III} (full 
\textsc{Veritas} without dual denoising), and the complete 
\textbf{\textsc{Veritas}}. As shown in 
Fig.~\ref{fig5}(a), removing either stage leads to a 
noticeable performance drop relative to the full model: 
removing Stage~II causes the most significant 
degradation, confirming that malicious node pruning is the 
primary defense mechanism; removing Stage~III also 
reduces accuracy, highlighting the importance of utility 
restoration. The full 
\textsc{Veritas} achieves the best performance, 
demonstrating that the two stages are complementary and 
both indispensable.

\begin{figure*}[t]
    \centering
     \begin{minipage}{\textwidth}
        \subfigure[\normalsize Ablation]{\includegraphics[width=0.245\linewidth]{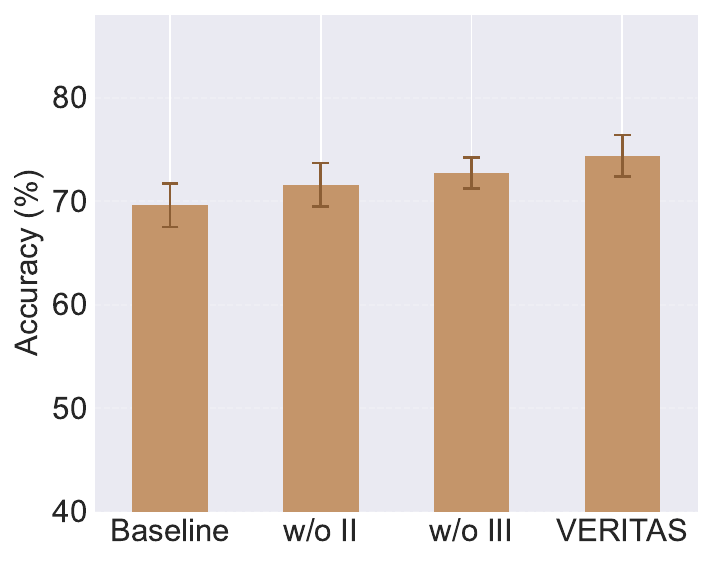}}
        \subfigure[\normalsize Parameter $\epsilon_x$]{\includegraphics[width=0.245\linewidth]{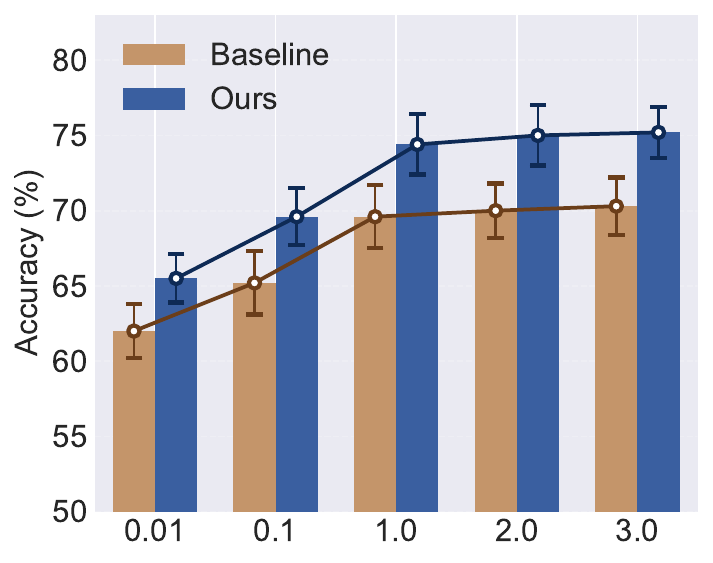}}
        \subfigure[\normalsize Parameter $\epsilon_a$]{\includegraphics[width=0.245\linewidth]{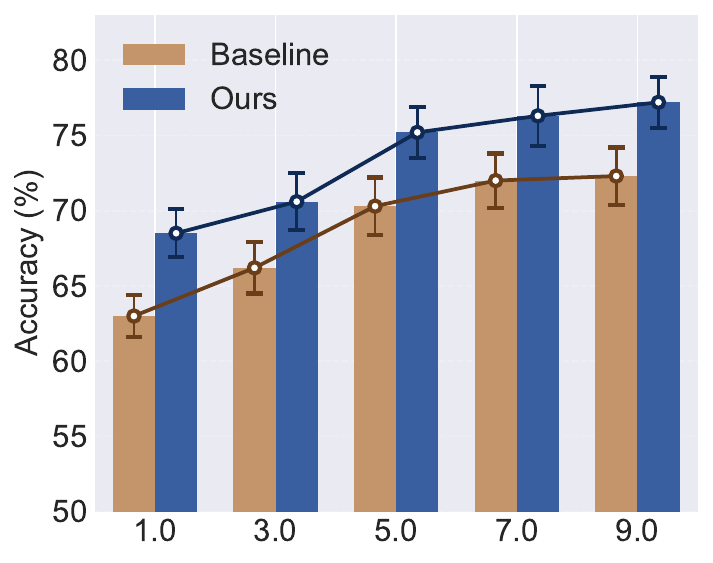}}
        \subfigure[\normalsize Parameter $r^\star$]{\includegraphics[width=0.245\linewidth]{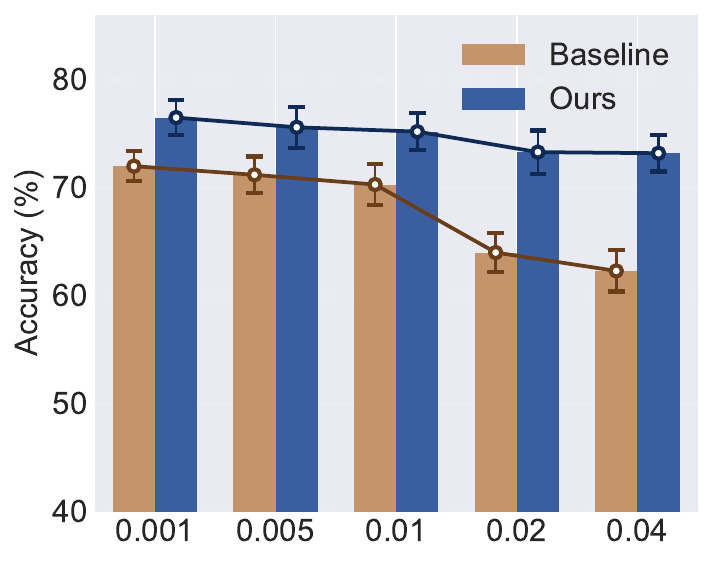}}
     \end{minipage}\\ 
   \caption{Analysis of \textsc{Veritas} under data poisoning attacks.
(a) Ablation study under different defense configurations, evaluating the contribution of each core component.
(b)–(d) Parameter sensitivity with respect to the feature privacy budget $\epsilon_x$, structural privacy budget $\epsilon_a$, and injection rate $r^\star$, respectively.}
    \label{fig5}
\end{figure*}

\subsubsection{Parameter Analysis}
We analyze the sensitivity of \textsc{Veritas} to three key parameters: the feature privacy budget $\epsilon_x$, the structural privacy budget $\epsilon_a$, and the injection rate $r^\star$. As shown in Fig.~\ref{fig5}(b) and (c), the node classification accuracy consistently improves as either $\epsilon_x$ or $\epsilon_a$ increases. This is because larger privacy budgets introduce less perturbation noise, preserving more informative feature and structural signals for malicious node detection. Fig.~\ref{fig5}(d) evaluates the effect of the $r^\star$. We observe that \textsc{Veritas} maintains stable performance when the proportion of malicious nodes is relatively low. Although the accuracy gradually decreases as $r^\star$ increases, the degradation remains moderate, indicating that \textsc{Veritas} can effectively withstand increasingly strong poisoning attacks.

\section{Related Work}\label{S2}

In this section, we review the existing literature related to local differential privacy ($\triangleright$~\ref{sec6.1}), graph neural networks ($\triangleright$~\ref{sec6.2}), and locally private graph learning ($\triangleright$~\ref{sec6.3}).

\subsection{Local Differential Privacy}\label{sec6.1}
LDP~\cite{kasiviswanathan2011can} has emerged as a \textit{de facto} standard for privacy-preserving data collection and analysis in decentralized environments, as it eliminates the need for a trusted data curator. In the LDP setting, each user perturbs their sensitive data before transmission, ensuring rigorous privacy guarantees even against an untrusted server. Owing to its strong privacy protection and practical deployability, LDP has attracted extensive attention from both academia and industry~\cite{xiong2020comprehensive,cormode2018privacy}. Representative mechanisms include randomized response~\cite{warner1965randomized,kairouz2016discrete} for categorical data, as well as the Laplace~\cite{phan2017adaptive}, Gaussian~\cite{balle2018improving}, piecewise~\cite{wang2019collecting}, multi-bit~\cite{sajadmanesh2021locally}, and square wave~\cite{li2020estimating} mechanisms for numerical data. Building upon these primitives, numerous studies have investigated private frequency estimation, distribution learning, and machine learning under LDP constraints. The practicality of LDP has also led to its large-scale deployment in real-world systems by major technology companies, including Google~\cite{erlingsson2014rappor}, Apple~\cite{patent1}, and Microsoft~\cite{ding2017collecting}. More recently, the scope of LDP has been extended from tabular data to graph-structured data~\cite{qin2017generating,zhu2023blink,lin2022towards}, laying the foundation for locally private graph learning.

\subsection{Graph Neural Networks}\label{sec6.2}
GNNs~\cite{DBLP:conf/iclr/KipfW17,wu2020comprehensive} have emerged as the dominant framework for learning over graph-structured data, achieving remarkable success in a wide range of graph mining tasks~\cite{DBLP:conf/iclr/KipfW17,liu2025hierarchical,tsitsulin2023graph}. By iteratively aggregating information from neighboring nodes, GNNs effectively capture both topological dependencies and node attributes, yielding expressive graph representations~\cite{DBLP:conf/iclr/VelickovicCCRLB18,DBLP:conf/iclr/KipfW17,hamilton2017inductive}. Consequently, GNNs have been widely adopted in numerous real-world applications~\cite{sharma2024survey,jain2023opinion,liu2021pick}. Despite their effectiveness, GNNs are frequently deployed on graphs containing sensitive user attributes and relational information, raising significant privacy concerns. Recent studies have shown that adversaries can exploit trained GNN models to infer private information, giving rise to various privacy attacks, including attribute inference~\cite{meng2023devil}, membership inference~\cite{zhang2022inference}, and link inference~\cite{wu2022linkteller}. These risks have motivated a growing body of research on privacy-preserving graph learning, among which locally private graph learning has attracted particular attention.

\subsection{Locally Private Graph Learning}\label{sec6.3}
Locally private graph learning extends the LDP paradigm to graph-structured data, enabling collaborative graph learning without revealing raw node features or graph topology to the server~\cite{sajadmanesh2021locally,lin2022towards,hegoing,he2025devil,hidano2022degree,li2024privacy,pei2023privacy,he2026devil,zhu2023blink,he2026differentially,he2026towards}. By leveraging the inherent noise tolerance of GNN message passing~\cite{DBLP:conf/iclr/KipfW17,sajadmanesh2021locally}, these methods can effectively learn useful node representations from perturbed graph data, maintaining strong utility under rigorous privacy constraints. Despite these advantages, the decentralized nature of locally private graph learning introduces unique security vulnerabilities. Recent studies~\cite{he2025devil} have shown that malicious participants can exploit both feature and structural perturbation channels to inject adversarial signals, which are subsequently propagated and amplified through the neighborhood aggregation process of GNNs, leading to significant utility degradation. Defending against such attacks is particularly challenging in the local privacy setting. Unlike centralized graph learning, where anomaly detection methods can leverage global graph statistics and feature correlations~\cite{zhang2020gnnguard,tang2020transferring,sun2022adversarial}, LDP mechanisms deliberately obfuscate both structural and attribute information, rendering many existing defenses ineffective. These challenges highlight a critical gap in the security of locally differentially private graph learning and underscore the urgent need for defense mechanisms specifically tailored to the LDP setting.

\section{Conclusion}

In this paper, we propose \textsc{Veritas}, a poisoning-resilient locally private graph learning protocol built upon a \textit{trust-but-verify} paradigm. By leveraging the asymmetry of bilateral attestation signals, \textsc{Veritas} effectively identifies and prunes malicious participants while preserving rigorous LDP guarantees. Extensive experiments on four datasets under multiple LDP mechanisms and GNN architectures demonstrate that \textsc{Veritas} consistently defends against poisoning attacks and significantly improves task utility. These results highlight the importance of integrating security considerations into locally private graph learning and establish a promising direction for building robust privacy-preserving graph learning systems.

\section*{Acknowledgment}

 This work was supported by the National Key Research and Development Program of China (No. 2024YFF0907401), the National Natural Science Foundation of China (No. 62072052), and the BUPT Excellent Ph.D. Students Foundation (No. CX20260030).

\bibliographystyle{IEEEtran}
\bibliography{paper}

\end{document}